\documentclass{article}

\pdfoutput=1

\usepackage{graphicx}
\usepackage{caption}
\usepackage{subcaption}
\usepackage{amsmath,amsfonts,amsthm,amssymb}
\usepackage{listings}
\usepackage{tabularx}
\usepackage{multicol}
\usepackage{float}
\usepackage[protrusion=true,expansion=true]{microtype}
\usepackage{wrapfig}
\usepackage[numbers]{natbib}

\usepackage[algoruled,vlined]{algorithm2e}
\usepackage{algorithmic}

\usepackage{multicol}
\usepackage{comment}

\usepackage{hyperref}

\newcommand{\term}[1]{\textbf{#1}}

\newcommand{\figref}[1]{Figure~\ref{#1}}

\newcommand{\secref}[1]{Section~\ref{#1}}

\newcommand{\thmref}[1]{Theorem~\ref{#1}}

\newcommand{\lemref}[1]{Lemma~\ref{#1}}

\newcommand{\algref}[1]{Alg.~\ref{#1}}

\newcommand{\appendref}[1]{Appendix~\ref{#1}}

\newcommand{\argmin}{\operatornamewithlimits{argmin}}

\newcommand{\given}{\,|\,}

\newcommand{\Ex}[1]{\mathbf{E} \left[ #1 \right] }

\newcommand{\set}[1]{\left\{#1\right\}}

\newcommand{\union}{\cup}

\newcommand{\size}[1]{\left| #1 \right|}

\newcommand{\BigO}[1]{O\hspace{-1pt}\left( #1 \right)}

\SetKwFor{ParForAll}{for}{do in parallel}{end}
\SetKwFunction{Map}{Map}
\SetKwFunction{Reduce}{Reduce}

\SetKwInput{Input}{Input}
\SetKwInput{Output}{Output}
\SetKwInput{SideEffect}{SideEffect}
\SetKwInput{Define}{Define}
\SetKwInput{Global}{Global}
\SetKwFor{DoWithProbability}{with probability}{}{}
\SetKwFunction{DPMeansOp}{DPMeansOp}
\SetKwFunction{DPValidate}{DPValidate}
\SetKwFunction{OFLValidate}{OFLValidate}
\SetKwFunction{BPMeansOp}{BPMeansOp}
\SetKwFunction{BPValidate}{BPValidate}
\SetKwFunction{NewClusters}{AcceptedClusters}

\SetKwFunction{Mean}{Mean}
\SetKwFunction{Ref}{Ref}

\newcommand{\clusters}{\mathcal{C}}
\newcommand{\gclusters}{\hat\clusters}
\newcommand{\newclusters}{\tilde\clusters}

\newcommand{\bregd}[2]{\left\|#1-#2\right\|}

\newcommand{\CFL}{C^{\text{FL}}}
\newcommand{\muFL}{\mu^{\text{FL}}}

\newtheorem{thm}{Theorem}[section]

\newtheorem{lem}[thm]{Lemma}
\newtheorem{prop}[thm]{Proposition}

\newtheorem{claim}{Claim}

\title{Optimistic Concurrency Control\\ for Distributed Unsupervised Learning}

\author{
  Xinghao Pan \\
  \texttt{xinghao@eecs.berkeley.edu} \\
  \and
  Joseph E. Gonzalez \\
  \texttt{jegonzal@eecs.berkeley.edu} \\
  \and
  Stefanie Jegelka \\
  \texttt{stefje@eecs.berkeley.edu} \\
  \and
  Tamara Broderick \\
  \texttt{tab@stat.berkeley.edu} \\
  \and
  Michael I. Jordan \\
  \texttt{jordan@cs.berkeley.edu}
}

\date{}

\begin{document}

\maketitle

\begin{abstract}
\label{sec:abstract}
Research on distributed machine learning algorithms has focused primarily
on one of two extremes---algorithms that obey strict concurrency constraints
or algorithms that obey few or no such constraints.  We consider an intermediate
alternative in which algorithms optimistically assume that conflicts are
unlikely and if conflicts do arise a conflict-resolution protocol is invoked.
We view this ``optimistic concurrency control'' paradigm as particularly
appropriate for large-scale machine learning algorithms, particularly in
the unsupervised setting.  We demonstrate our approach in three problem
areas: clustering, feature learning and online facility location.  We evaluate 
our methods via large-scale experiments in a cluster computing environment.
\end{abstract}

\label{sec:intro}

The desire to apply machine learning to increasingly larger datasets
has pushed the machine learning community to address the challenges of
distributed algorithm design: partitioning and coordinating computation
across the processing resources.  In many cases, when computing
statistics of iid data or transforming features, the computation
factors according to the data and coordination is only required during
aggregation.  For these \emph{embarrassingly parallel} tasks, the
machine learning community has embraced the map-reduce paradigm, which
provides a template for constructing distributed algorithms that are
fault tolerant, scalable, and \emph{easy to study}.

However, in pursuit of richer models, we often introduce statistical
dependencies that require more sophisticated algorithms (e.g.,
collapsed Gibbs sampling or coordinate ascent) which were developed
and studied in the \emph{serial} setting.  Because these algorithms
iteratively transform a global state, parallelization can be
challenging and often requires frequent and complex coordination.

Recent efforts to distribute these algorithms can be divided into two
primary approaches.  The \term{mutual exclusion} approach, adopted by
\cite{Gonzalez11} and \cite{Low12}, guarantees a \emph{serializable} execution
preserving the theoretical properties of the serial algorithm but at
the expense of parallelism and costly locking overhead.
Alternatively, in the \term{coordination-free} approach, proposed by
\cite{Recht11} and \cite{Ahmed12}, processors communicate frequently without
coordination minimizing the cost of contention but leading to
stochasticity, data-corruption, and requiring potentially complex
analysis to prove algorithm correctness.

In this paper we explore a third approach, \term{optimistic
  concurrency control} (OCC) \cite{kung1981:occ} which offers the
performance gains of the coordination-free approach while at the same
time ensuring a serializable execution and preserving the theoretical
properties of the serial algorithm.  Like the coordination-free
approach, OCC exploits the infrequency of data-corrupting operations.
However, instead of allowing occasional data-corruption, OCC detects
data-corrupting operations and applies correcting computation.  As a
consequence, OCC automatically ensures correctness, and the analysis
is only necessary to guarantee optimal scaling performance.

We apply OCC to distributed nonparametric unsupervised
learning---including but not limited to clustering---and implement
distributed versions of the DP-Means \cite{kulis2011}, BP-Means
\cite{broderick13}, and online facility location (OFL) algorithms.  We
demonstrate how to analyze OCC in the context of the DP-Means
algorithm and evaluate the empirical scalability of the OCC approach
on all three of the proposed algorithms.  The primary contributions of
this paper are:
\begin{enumerate}
  \setlength{\itemsep}{0pt}
\item Concurrency control approach to distributing unsupervised learning
  algorithms.
\item Reinterpretation of online nonparametric clustering in the form
  of facility location with approximation guarantees.
\item Analysis of optimistic concurrency control for unsupervised learning.
\item Application to feature modeling and clustering.
\end{enumerate}

\section{Optimistic Concurrency Control}
\label{sec:background}

Many machine learning algorithms iteratively transform some global
state (e.g., model parameters or variable assignment) giving the
illusion of serial dependencies between each operation.  However, due
to sparsity, exchangeability, and other symmetries, it is
often the case that many, \emph{but not all}, of the
state-transforming operations can be computed concurrently while still
preserving \term{serializability}: the equivalence to some serial
execution where individual operations have been reordered.

This opportunity for serializable concurrency forms the foundation of
distributed database systems.  For example, two customers may
concurrently make purchases exhausting the inventory of unrelated
products, but if they try to purchase the same product then we may
need to serialize their purchases to ensure sufficient inventory.  One
solution (\emph{mutual exclusion}) associates locks with each product
type
and forces each purchase of the same product to be processed serially.
This might work for an unpopular, rare product but if we are
interested in selling a popular product for which we have a large
inventory the serialization overhead could lead to unnecessarily slow
response times.  To address this problem, the database community has
adopted \term{optimistic concurrency control} (OCC)
\cite{kung1981:occ} in which the system tries to satisfy the customers
requests without locking and corrects transactions that could lead to
negative inventory (e.g., by forcing the customer to checkout again).

Optimistic concurrency control exploits situations where most
operations can execute concurrently without conflicting or violating
serial invariants in our program.  For example, given sufficient
inventory the order in which customers are satisfied is immaterial and
concurrent operations can be executed serially to yield the same final
result.  However, in the rare event that inventory is nearly depleted
two concurrent purchases may not be serializable since the inventory
can never be negative.  By shifting the cost of concurrency control to
rare events we can admit more costly concurrency control mechanisms
(e.g., re-computation) in exchange for an efficient, simple,
coordination-free execution for the majority of the events.

Formally, to apply OCC we must define a set of \term{transactions}
(i.e., operations or collections of operations), a mechanism to detect
when a transaction violates serialization invariants (i.e., cannot be
executed concurrently), and a method to correct (e.g., rollback)
transactions that violate the serialization invariants.  Optimistic
concurrency control is most effective when the cost of validating
concurrent transactions is small and conflicts occur infrequently.

Machine learning algorithms are ideal for optimistic concurrency
control.  The conditional independence structure and sparsity in our
models and data often leads to sparse parameter updates substantially 
reducing the chance of conflicts.  Similarly, symmetry in our models 
often provides the flexibility to reorder serial operations while 
preserving algorithm invariants.
Because the dependency structure is encoded in the model we can easily
detect when an operation violates serial invariants and correct by
rejecting the change and rerunning the computation. Alternatively, we
can exploit the semantics of the operations to resolve the conflict by
accepting a modified update. As a consequence OCC allows us to easily
construct provably correct and efficient distributed algorithms without
the need to develop new theoretical tools to analyze chaotic
convergence or non-deterministic distributed behavior.

\subsection{The OCC Pattern for Machine Learning}

Optimistic concurrency control can be distilled to a simple pattern
(meta-algorithm) for the design and implementation of distributed machine
learning systems.  We begin by evenly partitioning $N$ data points (and the
corresponding computation) across the $P$ available processors.  Each
processor maintains a replicated view of the global state and
\emph{serially} applies the learning algorithm as a sequence of
operations on its assigned data and \emph{the global state}.  If an
operation mutates the global state in a way that preserves the
serialization invariants then the operation is accepted locally and
its effect on the global state, if any, is eventually replicated to
other processors.

However, if an operation could potentially conflict with operations on
other processors then it is sent to a unique serializing processor
where it is rejected or corrected and the resulting global state
change is eventually replicated to the rest of the processors.
Meanwhile the originating processor either tentatively accepts the
state change (if a rollback operator is defined) or proceeds as though
the operation has been deferred to some point in the future.

While it is possible to execute this pattern asynchronously with
minimal coordination, for simplicity we adopt the bulk-synchronous
model of \cite{Valiant90} and divide the computation into
\emph{epochs}.  Within an epoch $t$, $b$ data points $\mathcal{B}(p,t)$ are evenly assigned
to each of the $P$ processors.
Any state changes or serialization operations are transmitted at the
end of the epoch and processed before the next epoch.
While potentially slower than an asynchronous execution, the
bulk-synchronous execution is deterministic and can be easily
expressed using existing systems like Hadoop or Spark \cite{zaharia2010:spark}.

\section{OCC for Unsupervised Learning}

Much of the existing literature on distributed machine learning
algorithms has focused on classification and regression problems,
where the underlying model is continuous.  In this paper we 
apply the OCC pattern to machine learning problems that have a more
discrete, combinatorial flavor---in particular unsupervised clustering
and latent feature learning problems. These problems exhibit symmetry
via their invariance to both data permutation and cluster or feature
permutation.  Together with the sparsity of interacting operations in
their existing serial algorithms, these problems offer a unique
opportunity to develop OCC algorithms.

The K-means algorithm provides a paradigm example; here the
inferential goal is to partition the data.
Rather than focusing solely on K-means, however, we have been inspired
by recent work in which a general family of K-means-like algorithms
have been obtained by taking Bayesian nonparametric (BNP) models based
on combinatorial stochastic processes such as the Dirichlet process,
the beta process, and hierarchical versions of these processes, and
subjecting them to \emph{small-variance asymptotics} where the
posterior probability under the BNP model is transformed into a cost
function that can be optimized~\cite{broderick13}.  The algorithms
considered to date in this literature have been developed and analyzed 
in the serial setting; our goal is to explore distributed algorithms for
optimizing these cost functions that preserve the structure and
analysis of their serial counterparts.

\subsection{OCC DP-Means}
\label{sec:dpmeans}

\begin{figure}[p]
  \footnotesize
  \centering
  \begin{multicols}{2}
    \begin{minipage}{0.51\textwidth}
      \begin{algorithm}[H]
        \DontPrintSemicolon
        \caption{Serial DP-means}
        \label{alg:dpm}
        \Input{data $\{x_i\}_{i=1}^N$, threshold $\lambda$}
        $\clusters \leftarrow \emptyset$ \;
        \While{not converged}{
          \For{$i$  = 1 to $N$}{
            $\mu^* \leftarrow \argmin_{\mu \in \clusters} \bregd{x_i}{\mu}$\;
            \If{$\bregd{x_i}{\mu^*} > \lambda$}{
              $z_i \leftarrow x_i $ \;
              $\clusters \leftarrow \clusters \union x_i$ \tcp*{New cluster} 
            }
            \lElse{  
              $z_i \leftarrow \mu^*$ \tcp*{Use nearest}
            }
          }    
          \For(\tcp*[h]{Recompute Centers}){$\mu \in \clusters$}{
            $\mu \leftarrow $ \Mean{$\set{x_i \given z_i = \mu}$ } \;
          }
        }
        \Output{Accepted cluster centers $\clusters$}
      \end{algorithm}
    \end{minipage}

 \begin{minipage}{0.49\textwidth}
      \begin{algorithm}[H]
        \DontPrintSemicolon
        \caption{\texttt{DPValidate} }
        \label{alg:dp_validate}
        \Input{Set of proposed cluster centers $\gclusters$}
        $\clusters \leftarrow \emptyset$ \;
        \For{$x \in \gclusters$}{
          $\mu^* \leftarrow \argmin_{\mu \in \clusters} \bregd{x}{\mu}$\;
          \If(\tcp*[h]{Reject}){$\bregd{x_i}{\mu^*} < \lambda$}{
           \Ref{$x$} $\leftarrow \mu^*$ \tcp*{Rollback Assgs} 
          }\lElse{
            $\clusters \leftarrow \clusters \union x$ \tcp*{Accept}
          }
        }
        \Output{Accepted cluster centers $\clusters$}
      \end{algorithm}
    \end{minipage}
  \end{multicols}

\begin{minipage}{1.0\textwidth}
      \begin{algorithm}[H]
        \DontPrintSemicolon
        \caption{Parallel DP-means}
        \label{alg:dpdist}
        \Input{data $\{x_i\}_{i=1}^N$, threshold $\lambda$}
        \Input{Epoch size $b$ and $P$ processors} 
        \Input{Partitioning $\mathcal{B}(p,t)$ of data 
          $\set{x_i}_{i \in \mathcal{B}(p,t)}$  to processor-epochs 
          where $b = \size{\mathcal{B}(p,t)}$}
        $\clusters \leftarrow \emptyset $ \;
        \While{not converged}{
          \For{epoch $t$ = 1 to $N/(Pb)$}{
            $\gclusters \leftarrow \emptyset$ \tcp*{New candidate centers}
            \ParForAll{$p \in \set{1, \ldots, P}$}{
              \tcp{Process local data}
              \For{$i \in \mathcal{B}(p,t)$} {
                $\mu^* \leftarrow \argmin_{\mu \in \clusters} \bregd{x_i}{\mu}$\;
                \tcp{Optimistic Transaction}
                \If{$ \bregd{x_i}{\mu^*} > \lambda$}{
                  $z_i \leftarrow$ \Ref{$x_i$}  \;
                  $\gclusters \leftarrow \gclusters \union x_i$ 
                }\lElse{  
                  $z_i \leftarrow \mu^*$ \tcp*{Always Safe}
                }
              }
            }
            \tcp{Serially validate clusters}
            $\clusters \leftarrow \clusters  \,\, \union \,\, \DPValidate(\gclusters)$ \;
          }
          \For(\tcp*[h]{Recompute Centers}){$\mu \in \clusters$}{
            $\mu \leftarrow $ \Mean{$\set{x_i \given z_i = \mu}$ } \;
          }
        }
        \Output{Accepted cluster centers $\clusters$}
      \end{algorithm}    
    \end{minipage}  

  \caption{ 
The Serial DP-Means algorithm and distributed implementation using the
OCC pattern. 
}
  \label{fig:dpmeans}
\end{figure}

We first consider the \emph{DP-means} algorithm (\algref{alg:dpm})
introduced by~\cite{kulis2011}.
Like the K-means algorithm, DP-Means alternates between updating the
cluster assignment $z_i$ for each point $x_i$ and recomputing the
centroids $\clusters = \set{\mu_k}_{k=1}^K$ associated with each
clusters.  However, DP-Means differs in that the number of clusters is
not fixed a priori.  Instead, if the distance from a given data point
to all existing cluster centroids is greater than a parameter
$\lambda$, then a new cluster is created.  While the second phase is
trivially parallel, the process of introducing clusters in the first
phase is inherently serial.  However, clusters tend to be introduced
infrequently, and thus DP-Means provides an opportunity for OCC.

In \algref{alg:dpdist} we present an OCC parallelization of the
DP-Means algorithm in which each iteration of the serial DP-Means
algorithm is divided into $N/(Pb)$ bulk-synchronous epochs.  The data 
is evenly partitioned $\set{x_i}_{i \in \mathcal{B}(p,t)}$ across
processor-epochs into blocks of size $b = \size{\mathcal{B}(p,t)}$.
During each epoch $t$, each processor $p$ evaluates the cluster
membership of its assigned data $\set{x_i}_{i \in \mathcal{B}(p,t)}$
using the cluster centers $\clusters$ from the previous epoch and
optimistically proposes a new set of cluster centers $\gclusters$.
At the end of each epoch the proposed cluster centers, $\gclusters$,
are \emph{serially} validated using \algref{alg:dp_validate}.  The
validation process accepts cluster centers that are not covered by 
(i.e., not within $\lambda$ of)
already accepted cluster centers.  When a cluster center is rejected
we update its reference to point to the already accepted center,
thereby correcting the original point assignment.

\subsection{OCC Facility Location}
\label{sec:ofl}

The DP-Means objective turns out to be equivalent to the classic Facility
Location (FL) objective:
$$  J(\clusters) = \sum_{x \in X}\min_{\mu \in \clusters}\bregd{x}{\mu}^2 + \lambda^2 |\clusters|,
$$
which selects the set of cluster centers (facilities) $\mu \in \clusters$
that minimizes the shortest distance $\bregd{x}{\mu}$ to each point
(customer) $x$ as well as the penalized cost of the clusters $\lambda^2
\size{\clusters}$.  However, while DP-Means allows the clusters to be
arbitrary points (e.g., $\clusters \in \mathbb{R}^D$), FL constrains
the clusters to be points $\clusters \subseteq \mathcal{F}$ in a set
of candidate locations $\mathcal{F}$.
Hence, we obtain a link between combinatorial Bayesian models and FL
allowing us to apply algorithms with known approximation bounds to
Bayesian inspired nonparametric models.  As we will see in
\secref{sec:theory}, our OCC algorithm provides constant-factor
approximations for both FL and DP-means. 

Facility location has been studied intensely. We build on the
\emph{online} facility location (OFL) algorithm described by
Meyerson~\cite{Meyerson01}.
The OFL algorithm processes each data point $x$ serially
in a single pass by either adding $x$ to the set of clusters with
probability $\min(1,\min_{\mu \in \clusters}\bregd{x}{\mu}^2/\lambda^2)$ or
assigning $x$ to the nearest existing cluster.  Using OCC we are able
to construct a distributed OFL algorithm (\algref{alg:occ_ofl}) which is
nearly identical to the OCC DP-Means algorithm (\algref{alg:dpdist})
but which provides strong approximation bounds.  The OCC OFL
algorithm differs only in that clusters are introduced and validated
stochastically---the validation process ensures that the new clusters 
are accepted with probability equal to the serial algorithm.

\begin{figure} 
  \footnotesize
  \centering
  \begin{multicols}{2}
    \begin{minipage}[t]{0.49\textwidth}
      \begin{algorithm}[H]
        \DontPrintSemicolon
        \caption{Parallel OFL}
        \label{alg:occ_ofl}
        \Input{Same as DP-Means}
        \For($\gclusters \leftarrow \emptyset$){epoch $t$ = 1 to $N/(Pb)$}{
          \ParForAll{$p \in \set{1, \ldots, P}$}{
            \For{$i \in \mathcal{B}(p,t)$} {
              $d \leftarrow \min_{\mu \in \clusters} \bregd{x_i}{\mu}$ \;
              \DoWithProbability{$\min\set{d^2,\lambda^2}/\lambda^2$}{
                $\gclusters \leftarrow \gclusters \union (x_i,d)$\;
              }
            }
          }
          $\clusters \leftarrow \clusters  \,\, \union \,\, \OFLValidate(\gclusters)$ \;
        }
        \Output{Accepted cluster centers $\clusters$}
      \end{algorithm}
    \end{minipage}
    \begin{minipage}[t]{0.49\textwidth}
      \begin{algorithm}[H]
        \DontPrintSemicolon
        \caption{OFLValidate}
        \label{alg:ofl_validate}
        \Input{Set of proposed cluster centers $\gclusters$}
        $\clusters \leftarrow \emptyset$ \;
        \For{$(x,d) \in \gclusters$}{
          $d^* \leftarrow \min_{\mu \in \clusters} \bregd{x}{\mu}$\;
          \DoWithProbability{  $\min\set{d^{*2}, d^2}/d^2$ }{
            $\clusters \leftarrow \clusters \union x$ \tcp*{Accept}
          }
        }
        \Output{Accepted cluster centers $\clusters$} 
      \end{algorithm}
    \end{minipage}
  \end{multicols}
  \caption{The OCC algorithm for Online Facility Location (OFL).}
\end{figure}

\subsection{OCC BP-Means}
\label{ssec:bpmeans}

\emph{BP-means} is an algorithm for learning collections of latent binary
\emph{features}, providing a way to define groupings of data points that 
need not be mutually exclusive or exhaustive like clusters.

As with serial DP-means,
there are two phases in serial BP-means (\algref{alg:bpm}).
In the first phase, each data point $x_i$
is labeled with binary assignments from a collection of features
($z_{ik} = 0$ if $x_i$ doesn't belong to feature $k$; otherwise $z_{ik} = 1$)
to construct a representation $x_i\approx\sum_kz_{ik}f_k$.
In the second phase, parameter values (the feature means $f_{k} \in \gclusters$)
are updated based on the assignments.
The first step also includes the possibility of introducing an
additional feature. While the second phase is trivially parallel, 
the inherently serial nature of the first phase combined with the infrequent
introduction of new features points to the usefulness of OCC
in this domain.

The OCC parallelization for BP-means follows the same basic structure as OCC DP-means.
Each transaction operates on a data point $x_i$ in two phases.
In the first, analysis phase, the optimal representation $\sum_k z_{ik}f_k$ is found.
If $x_i$ is not well represented (i.e., $\|x_i-\sum_k z_{ik}f_k\|>\lambda$), the difference is proposed as a new feature in the second validation phase.
At the end of epoch $t$, the proposed features $\set{f^{new}_i}$ are serially validated to obtain a set of accepted features $\newclusters$.
For each proposed feature $f_i^{new}$, the validation process first finds the optimal representation $f^{new}_i\approx\sum_{f_k\in\newclusters}z_{ik}f_k$ using \textit{newly accepted features}.
If $f_i^{new}$ is not well represented, the difference $f_i^{new}-\sum_{f_k\in\newclusters}z_{ik}f_k$ is added to $\newclusters$ and accepted as a new feature.

Finally, to update the feature means, let $F$ be the $K$-row
matrix of feature means. The feature means update
$F\leftarrow (Z^TZ)^{-1}Z^TX$ can be evaluated
as a single
transaction by computing the sums
$Z^TZ=\sum_i z_i z_i^T$ (where $z_i$
is a $K \times 1$ column vector so $z_i z_i^T$ is a $K \times K$ matrix)
and $Z^TX=\sum_i z_i x_i^T$ in parallel.

We present the pseudocode for the OCC parallelization of BP-means in Appendix \ref{sec:appendix_algorithms_bpmeans}.

\section{Analysis of Correctness and Scalability}
\label{sec:theory}

We now establish the correctness and scalability of the proposed OCC
algorithms. In contrast to the coordination-free pattern in which
scalability is trivial and correctness often requires strong
assumptions or holds only in expectation, the OCC pattern leads to simple
proofs of correctness and challenging scalability analysis. However,
in many cases it is preferable to have algorithms that are correct and
probably fast rather than fast and possibly correct. 
We first establish serializability:
\begin{thm}[Serializability]\label{thm:serial}
  The distributed DP-means, OFL, and BP-means algorithms are
  serially equivalent to DP-means, OFL and BP-means, respectively.
\end{thm}
The proof (\appendref{ssec:appendix_serialization}) of
\thmref{thm:serial} is relatively straightforward and is obtained by
constructing a permutation function that describes an equivalent
serial execution for each distributed execution. The proof can easily 
be extended to many other machine learning algorithms.

Serializability allows us to easily extend important theoretical
properties of the serial algorithm to the distributed setting.
For example, by invoking serializability, we can establish the
following result for the OCC version of the online facility location
(OFL) algorithm:
\begin{lem}
  \label{lem:approxbd}
  If the data is randomly ordered, then
  the OCC OFL algorithm provides a constant-factor
  approximation for the DP-means objective.
  If the data is adversarially ordered, then OCC OFL provides a
  log-factor approximation to the DP-means objective.
\end{lem}
The proof (\appendref{ssec:appendix_serialization}) of
\lemref{lem:approxbd} is first derived in the serial setting then
extended to the distributed setting through serializability. In
contrast to divide-and-conquer schemes, whose approximation bounds
commonly depend \emph{multiplicatively} on the number of levels \cite{meyerson03},
\lemref{lem:approxbd} is unaffected by distributed processing and has
no communication or coarsening tradeoffs. Furthermore, to retain the
same factors as a batch algorithm on the full data, divide-and-conquer
schemes need a large number of preliminary centers at lower levels \citep{meyerson03,ailon09}. In
that case, the communication cost can be high, since all proposed clusters are sent at the same time, as opposed to the OCC approach.
We address the communication overhead (the number of
rejections) for our scheme next.

\paragraph{Scalability}

The scalability of the OCC algorithms depends on the number of
transactions that
are rejected during validation (i.e., the rejection rate).
While a general scalability analysis can be challenging, it is often
possible to gain some insight into the asymptotic dependencies by
making simplifying assumptions. In contrast to the coordination-free
approach, we can still \emph{safely} apply OCC algorithms in the
absence of a scalability analysis or when simplifying assumptions do
not hold.

To illustrate the techniques employed in OCC scalability analysis we
study the DP-Means algorithm. The scalability limiting factor of the
DP-Means algorithm is determined by the number of points that must be
serially validated. In the following theorem we show that the
communication cost only depends on the number of clusters and
processing resources and does not directly depend on the number of
data points. The proof is in App.~\ref{ssec:appendix_dpmeans}.
\begin{thm}[DP-Means Scalability] 
\label{thm:dp_rate}
Assume $N$ data points are generated iid to form
a random number ($K_N$) of well-spaced clusters of diameter $\lambda$:
$\lambda$ is an upper bound on the distances within clusters
and a lower bound on the distance between clusters.
Then the expected number of
\emph{serially} validated points is bounded above by $P b + \Ex{K_N}$
for $P$ processors and $b$ points per epoch.
\end{thm}
Under the separation assumptions of the theorem, the number of clusters present in $N$
data points, $K_{N}$, is exactly equal to the number of clusters found
by DP-Means in $N$ data points; call this latter quantity $k_{N}$.
The
experimental results in \figref{fig:sim_coupon0} suggest that
the bound of $Pb+k_N$ may hold more generally beyond the assumptions above. Since the master must
process at least $k_N$ points, the overhead caused by rejections is $P
b$ and independent of $N$.

To analyze the total running time, we note that after each of the
$N/(Pb)$ epochs the master and workers must communicate. Each worker
must process $N/P$ data points, and the master sees at most $k_N + Pb$
points. Thus, the total expected running time is $\BigO{N/(Pb) + N/P +
  Pb}$.

\section{Evaluation}
\label{sec:evaluation}

For our experiments, we generated synthetic data for clustering (DP-means and OFL) 
and feature modeling (BP-means).  The cluster and feature proportions were generated 
nonparametrically as described below.  All data points were generated in 
$\mathbb{R}^{16}$ space.  The threshold parameter $\lambda$ was fixed at 1.

\textbf{Clustering}:
The cluster proportions and indicators were generated simultaneously using the 
stick-breaking procedure for Dirichlet processes---`sticks' are `broken' on-the-fly 
to generate new clusters as and when necessary.\footnote{We chose to use stick-breaking procedures because the Chinese restaurant and Indian buffet processes are inherently sequential.
Stick-breaking procedures can be distributed by either truncation, or using OCC!}
For our experiments, we used a fixed concentration parameter $\theta=1$.
Cluster means were sampled $\mu_k\sim N(0,I_{16})$, and data points were generated at $x_i\sim N(\mu_{z_i},\frac{1}{4}I_{16})$.

\textbf{Feature modeling}:
We use the stick-breaking procedure of \cite{paisley2012:stick} to generate feature weights.
Unlike with Dirichlet processes, we are unable to perform stick-breaking on-the-fly with 
Beta processes.  Instead, we generate enough features so that with high probability 
$(>0.9999)$ the remaining non-generated features will have negligible weights $(<0.0001)$.
The concentration parameter was also fixed at $\theta=1$.  We generated feature means 
$f_k\sim N(0,I_{16})$ and data points $x_i\sim N(\sum_k z_{ik}f_k,\frac{1}{4}I_{16})$.

\subsection{Simulated experiments}
\label{ssec:evaluation_simulated}

\begin{figure}[ht]
  \centering
  \begin{subfigure}[b]{0.32\textwidth}
  	\centering
	  \includegraphics[width=140pt]{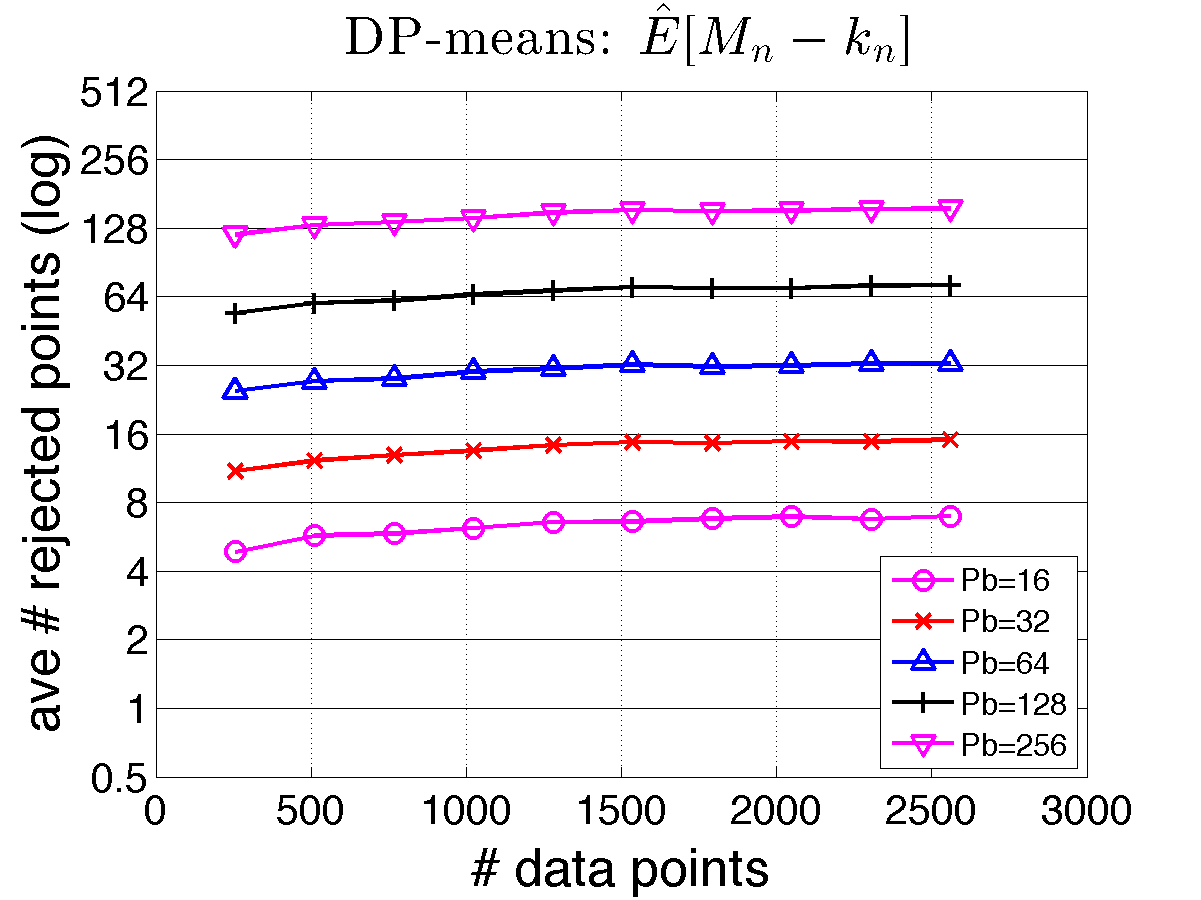}
  	\caption{\footnotesize OCC DP-means}
	  \label{fig:dpm_sim_coupon0}
	\end{subfigure}
  \begin{subfigure}[b]{0.32\textwidth}
  	\centering
	  \includegraphics[width=140pt]{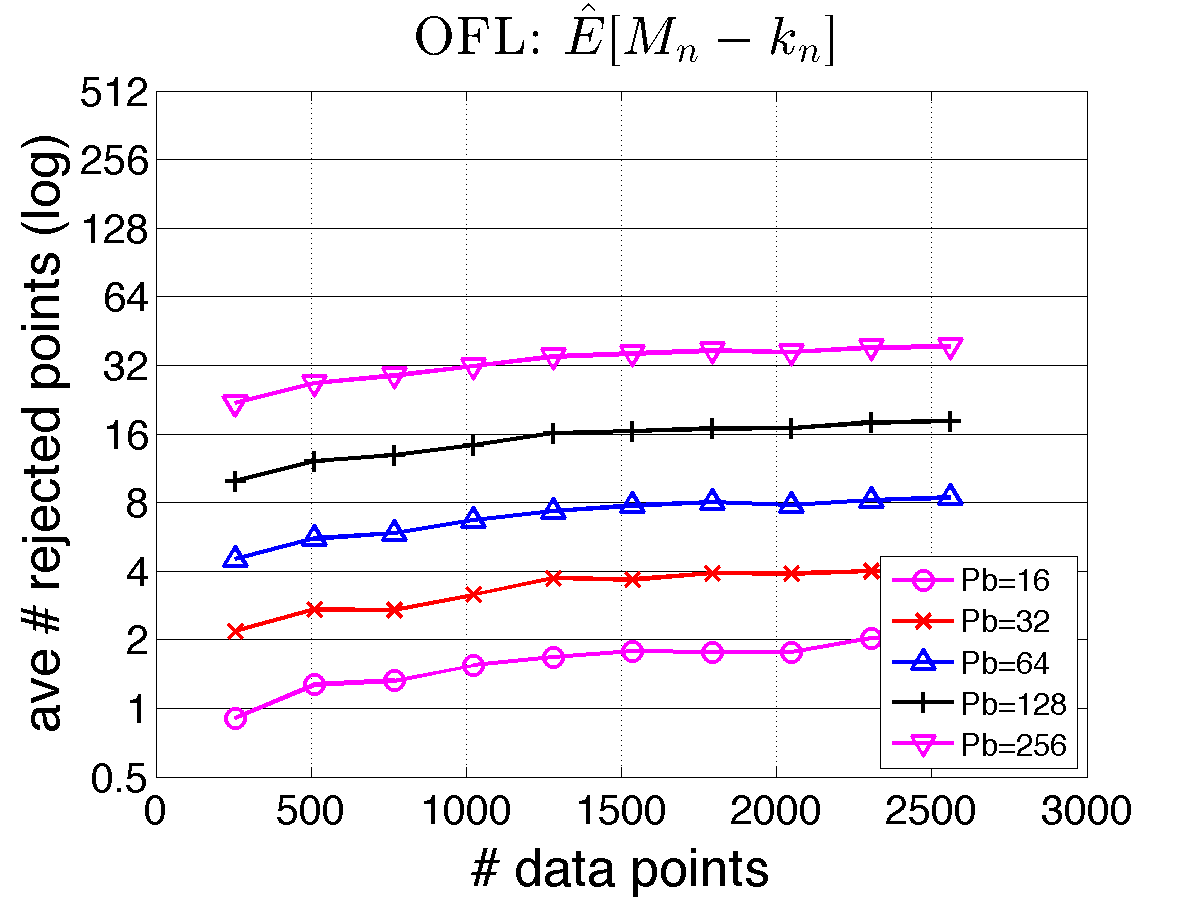}
  	\caption{\footnotesize OCC OFL}
	  \label{fig:ofl_sim_coupon0}
	\end{subfigure}
  \begin{subfigure}[b]{0.32\textwidth}
  	\centering
	  \includegraphics[width=140pt]{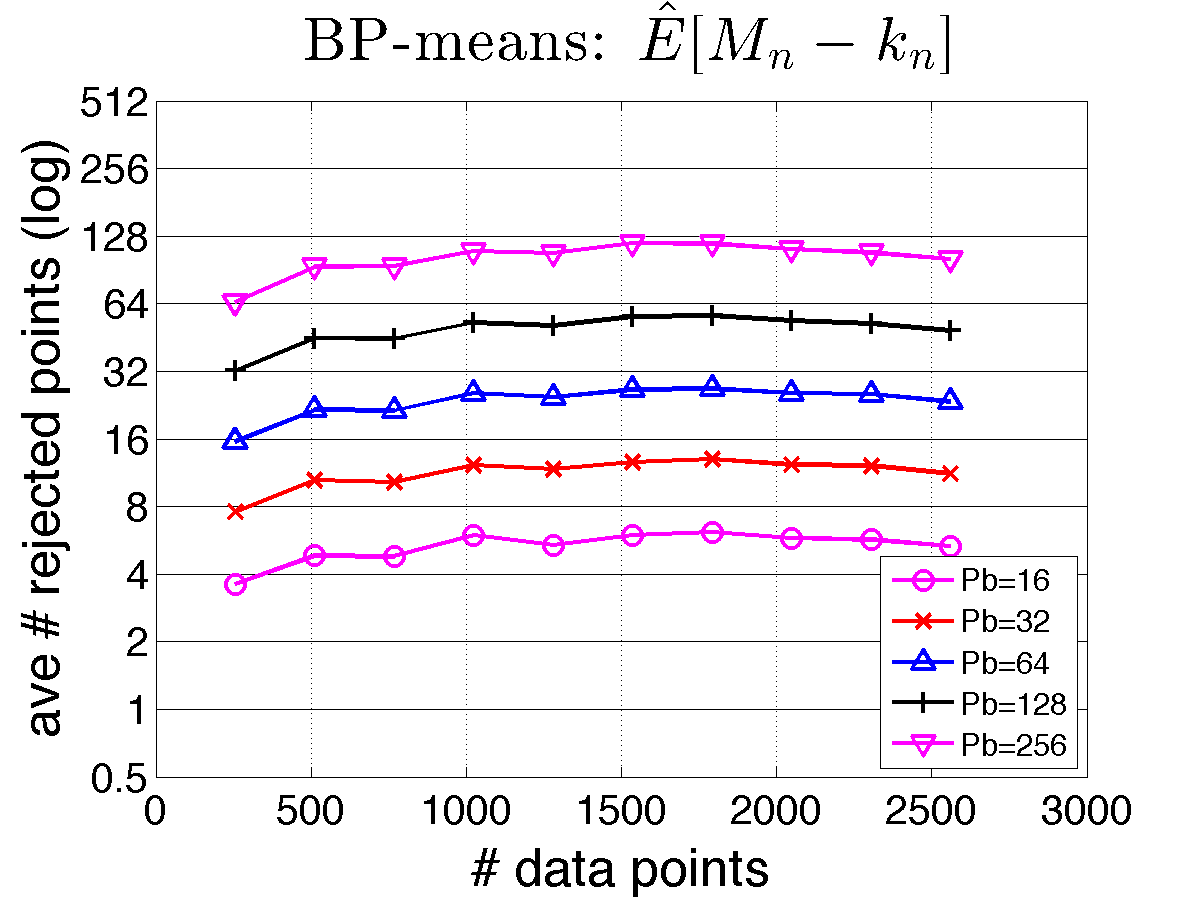}
  	\caption{\footnotesize OCC BP-means}
	  \label{fig:bpm_sim_coupon0}
	\end{subfigure}
	\caption{\footnotesize Simulated distributed DP-means, OFL and BP-means: expected number of data points proposed but not accepted as new clusters / features is independent of size of data set.}
	\label{fig:sim_coupon0}
\end{figure}

To test the efficiency of our algorithms, we simulated the first iteration (one complete pass over all the data, where most clusters / features are created and thus greatest coordination is needed) of each algorithm in MATLAB.
The number of data points, $N$, was varied from 256 to 2560 in intervals of 256.
We also varied $Pb$, the number of data points processed in one epoch, from 16 to 256 in powers of 2.
For each value of $N$ and $Pb$, we empirically measured $k_N$, the number of accepted clusters / features, and $M_N$, the number of proposed clusters / features.
This was repeated 400 times to obtain the empirical average $\hat{\mathbb{E}}[M_N-k_N]$,
the number of rejections.

For OCC DP-means, we observe $\hat{\mathbb{E}}[M_N-k_N]$ is bounded above by $Pb$ (Fig.~\ref{fig:dpm_sim_coupon0}), and that this bound is independent of the data set size, even when the assumptions of Thm \ref{thm:dp_rate} are violated.
(We also verified that similar empirical results are obtained when the assumptions are not violated; see Appendix \ref{ssec:appendix_dpmeans}.)
As shown in Fig.~\ref{fig:ofl_sim_coupon0} and Fig.~\ref{fig:bpm_sim_coupon0} 
the same behavior is observed for the OCC OFL and OCC BP-means algorithms.

\subsection{Distributed implementation and experiments}
\label{ssec:evaluation_spark}

We also implemented the distributed algorithms in Spark \cite{zaharia2010:spark}, 
an open-source cluster computing system.
The DP-means and BP-means algorithms were bootstrapped by pre-processing a small 
number of data points (1/16 of the first $Pb$ points)---this 
reduces the number of data points sent to the master on the first epoch, while still preserving serializability of the algorithms.
Our Spark implementations were tested on Amazon EC2 by processing a fixed data set on 1, 2, 4, 8 m2.4xlarge (memory-optimized, quadruple extra large, with 8 virtual cores and 64.8GiB memory) instances.

Ideally, to process the same amount of data, an algorithm and implementation with perfect scaling would take half the runtime on 8 machines as it would on 4, and so on.
The plots in Figure \ref{fig:spark_norm} shows this comparison by dividing all runtimes by the runtime on one machine.

\textbf{DP-means}: We ran the distributed DP-means algorithm on $2^{27}\approx134$M data points, using $\lambda=2$.
The block size $b$ was chosen to keep $Pb=2^{23}\approx8$M constant.
The algorithm was run for 5 iterations (complete pass over all data in 16 epochs).
We were able to get perfect scaling (Figure \ref{fig:dpm_spark_norm}) in all but the first iteration, when the master has to perform the most synchronization of proposed centers.

\textbf{OFL}: The distributed OFL algorithm was run on $2^{20}\approx1$M data points, using $\lambda=2$.
Unlike DP-means and BP-means, we did not perform bootstrapping.
Also, OFL is a single pass (one iteration) algorithm.
The block size $b$ was chosen such that $Pb=2^{16}\approx66$K data points are processed each epoch, which gives us 16 epochs.
Figure \ref{fig:ofl_spark_norm} shows that we get no scaling in the first epoch, where all the work is performed by the master processing all $Pb$ data points.
In later epochs, the master's workload decreases as fewer data points are proposed, but the workers' workload increases as the total number of centers increases.
Thus, scaling improves in the later epochs.

\textbf{BP-means}: Distributed BP-means was run on $2^{23}\approx8$M data points, with $\lambda=1$; block size was chosen such that $Pb=2^{19}\approx0.5$M is constant.
Five iterations were run, with 16 epochs per iteration.
As with DP-means, we were able to achieve nearly perfect scaling; see Figure \ref{fig:bpm_spark_norm}.

\begin{figure}[t]
  \centering
  \begin{subfigure}[b]{0.32\textwidth}
  	\centering
	  \includegraphics[width=140pt]{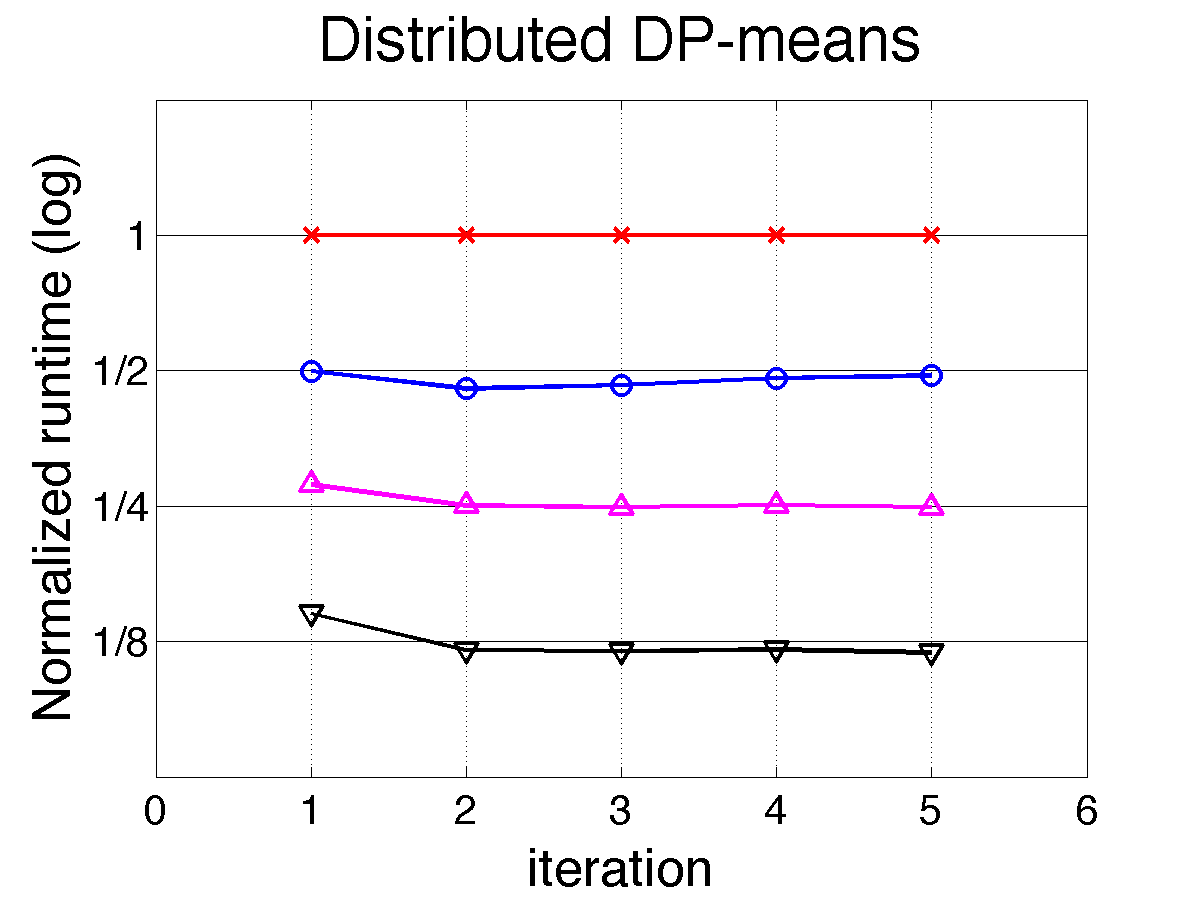}
  	\caption{\footnotesize OCC DP-means}
	  \label{fig:dpm_spark_norm}
	\end{subfigure}
  \begin{subfigure}[b]{0.32\textwidth}
  	\centering
	  \includegraphics[width=140pt]{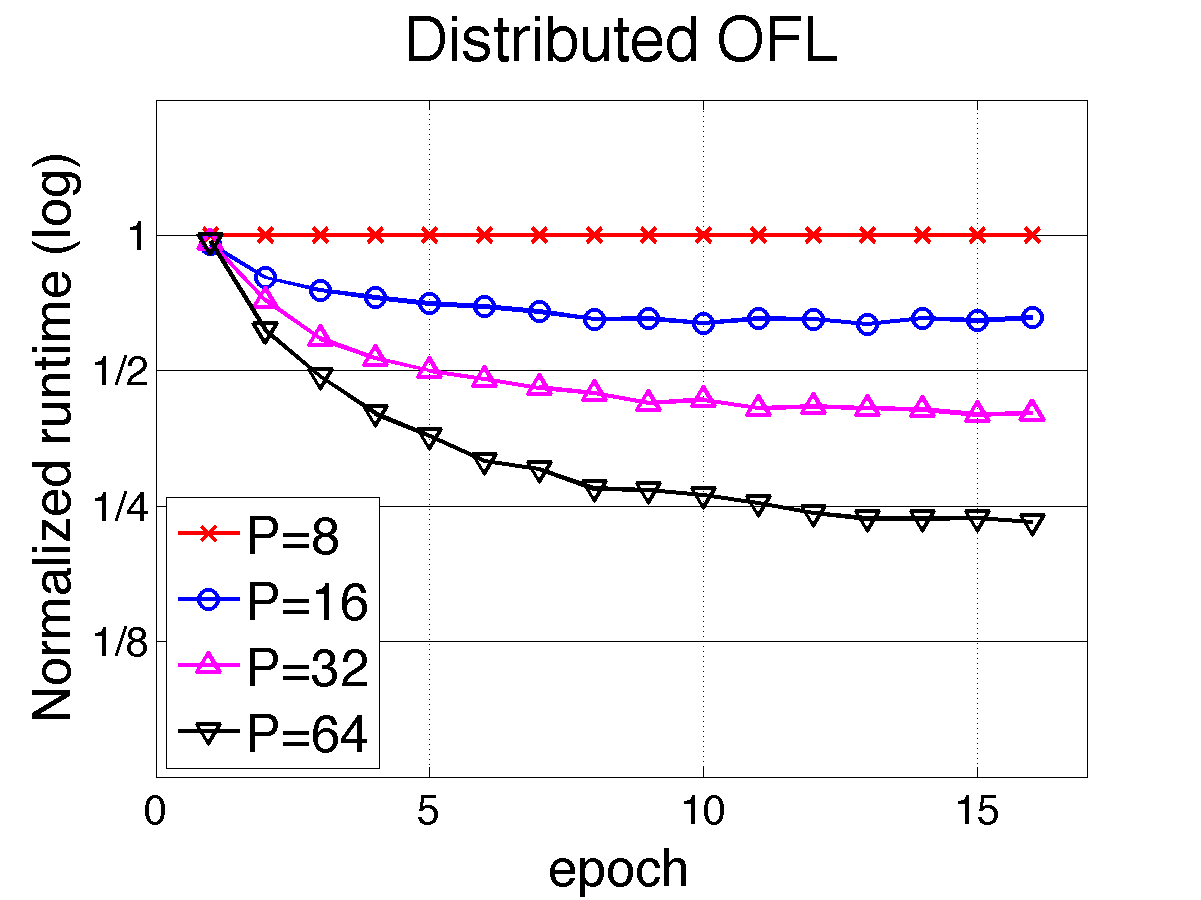}
  	\caption{\footnotesize OCC OFL}
	  \label{fig:ofl_spark_norm}
	\end{subfigure}
  \begin{subfigure}[b]{0.32\textwidth}
  	\centering
	  \includegraphics[width=140pt]{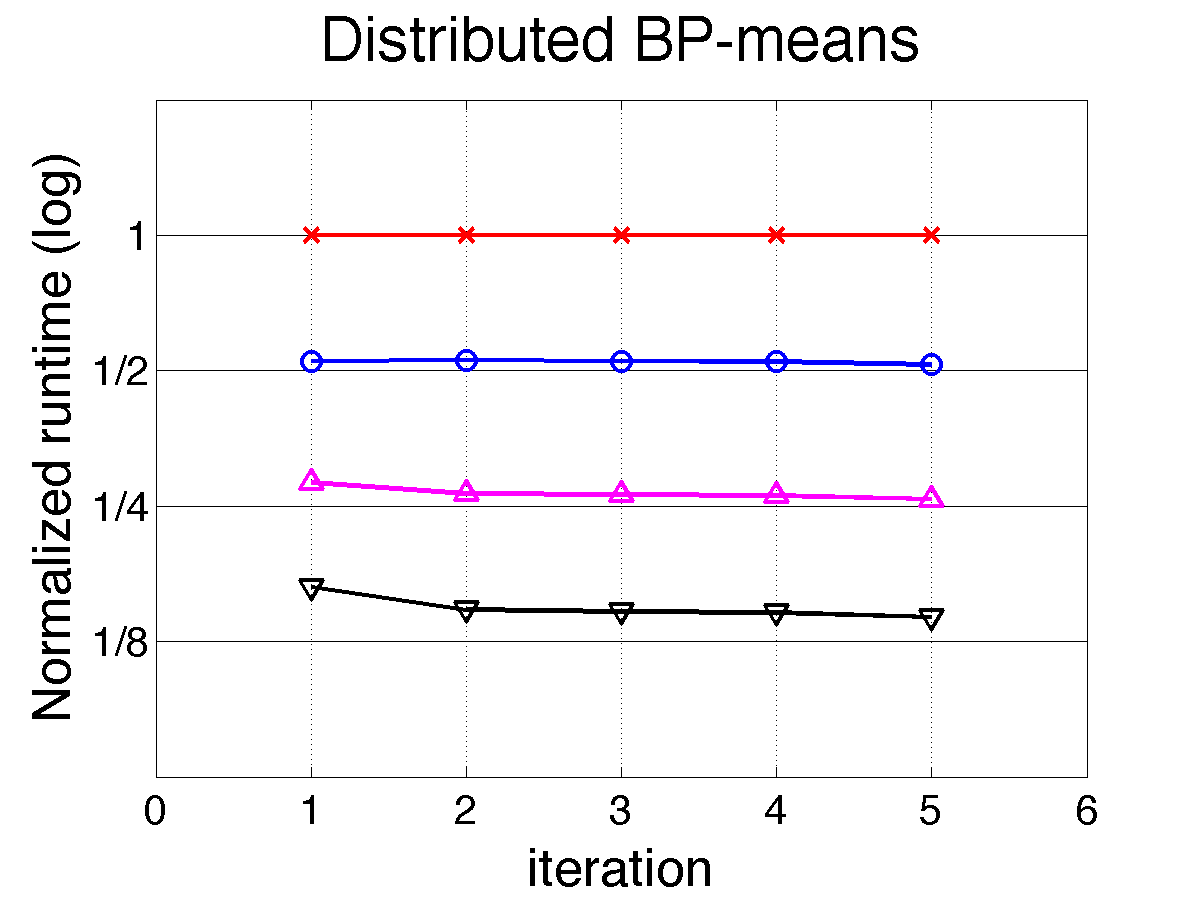}
  	\caption{\footnotesize OCC BP-means}
	  \label{fig:bpm_spark_norm}
	\end{subfigure}
	\caption{\footnotesize Normalized runtime for distributed algorithms. Runtime of each iteration / epoch is divided by that using 1 machine ($P=8$).
	Ideally, the runtime with 2, 4, 8 machines ($P=16,32,64$) should be respectively 1/2, 1/4, 1/8 of the runtime using 1 machine.
	OCC DP-means and BP-means obtain nearly perfect scaling for all iterations.
	OCC OFL rejects a lot initially, but quickly gets better in later epochs.}
	\label{fig:spark_norm}
\end{figure}

\section{Related work}
\label{sec:relatedwork}

Others have proposed alternatives to mutual exclusion and 
coordination-free parallelism for machine learning algorithm design.
Newman~\cite{Newman07} proposed transforming the underlying model to
expose additional parallelism while preserving the marginal
posterior. However, such constructions can be challenging or
infeasible and many hinder mixing or convergence.  Likewise,
Lovell~\cite{Lovell13} proposed a reparameterization of the underlying
model to expose additional parallelism through conditional
independence.

Additional work similar in spirit to ours using OCC-like techniques 
includes~\citet{Doshi09} who proposed an approximate parallel sampling 
algorithm for the IBP which is made exact by introducing an additional 
Metropolis-Hastings step, and~\citet{Xu11} who proposed a look-ahead 
strategy in which future samples are computed optimistically based on 
the likely outcomes of current samples.

A great amount of work addresses scalable clustering algorithms \cite{dhillon00,das07,ene11}.  
Many algorithms with provable approximation factors are streaming algorithms \citep{meyerson03,shindler11, charikar03} 
and inherently use hierarchies, or related divide-and-conquer approaches \citep{ailon09}. The approximation factors in such algorithms multiply across levels \citep{meyerson03}, and demand a careful tradeoff between communication and approximation quality that is obviated in our framework.
Other approaches use core sets \cite{badoiu02,feldman11}. A lot of methods \cite{ailon09,bahmani12,shindler11} first collect a set of centers and then re-cluster them, and therefore need to communicate all intermediate centers. Our approach avoids that, since a center causes no rejections in the epochs after it is established: the rejection rate does not grow with $K$. Still, as our examples demonstrate, our OCC framework can easily integrate and exploit many of the ideas in the cited works.

\section{Discussion}
\label{sec:conclusion}

In this paper we have shown how optimistic concurrency control can
be usefully employed in the design of distributed machine learning 
algorithms. As opposed to previous approaches, this preserves correctness, in most cases at a small cost.
We established the equivalence of our distributed OCC 
DP-means, OFL and BP-means algorithms to their serial counterparts, thus preserving 
their theoretical properties.  In particular, the strong approximation 
guarantees of serial OFL translate immediately to the distributed algorithm.
Our theoretical analysis ensures OCC DP-means achieves high parallelism 
without sacrificing correctness.  We implemented and evaluated all 
three OCC algorithms on a distributed computing platform and demonstrate 
strong scalability in practice.

We believe that there is much more to do in this vein.
Indeed, machine learning algorithms have many properties that distinguish
them from classical database operations and may allow going beyond the
classic formulation of optimistic concurrency control.  In particular
we may be able to partially or \emph{probabilistically} accept non-serializable
operations in a way that preserves underlying algorithm invariants.  
Laws of large numbers and concentration theorems may provide tools for 
designing such operations.  Moreover, the conflict detection mechanism 
can be treated as a control knob, allowing us to softly switch between 
stable, theoretically sound algorithms and potentially faster coordination-free 
algorithms.

\subsection*{Acknowledgments}
This research is supported in part by NSF CISE Expeditions award CCF-1139158 and DARPA XData Award FA8750-12-2-0331, and  gifts from Amazon Web Services, Google, SAP,  Blue Goji, Cisco, Clearstory Data, Cloudera, Ericsson, Facebook, General Electric, Hortonworks, Intel, Microsoft, NetApp, Oracle, Samsung, Splunk, VMware and Yahoo!.
This material is also based upon work supported in part by the Office of
Naval Research under contract/grant number N00014-11-1-0688. 
X. Pan's work is also supported in part by a DSO National Laboratories Postgraduate Scholarship.
T. Broderick's work is supported by a Berkeley Fellowship.

\bibliographystyle{abbrvnat}
\bibliography{references}

\newpage
\appendix

\section{Pseudocode for OCC BP-means}
\label{sec:appendix_algorithms_bpmeans}

Here we show the Serial BP-Means algorithm (\algref{alg:bpm}) and a parallel implementation of BP-means using the OCC pattern (\algref{alg:bpdist} and \algref{alg:bp_validate}), similar to OCC DP-means.
Instead of proposing new clusters centered at the data point $x_i$, in OCC BP-means we propose features $f_i^{new}$ that allow us to obtain perfect representations of the data point.
The validation process continues to improve on the representation $x_i\approx\sum_kz_{ik}f_k$ by using the most recently accepted features $f_{k'}\in\gclusters$, and only accepts a proposed feature if the data point is still not well-represented.

 \begin{figure}[h]
      \centering
      {\small
      \begin{algorithm}[H]
        \DontPrintSemicolon
        \caption{Parallel BP-means}
        \label{alg:bpdist}
        \Input{data $\{x_i\}_{i=1}^N$, threshold $\lambda$}
        \Input{Epoch size $b$ and $P$ processors} \Input{Partitioning
          $\mathcal{B}(p,t)$ of data $\set{x_i}_{i \in
            \mathcal{B}(p,t)}$ to processor-epochs where $b =
          \size{\mathcal{B}(p,t)}$} $\clusters \leftarrow \emptyset $
        \; \While{not converged}{ \For{epoch $t$ = 1 to $N/(Pb)$}{
            $\gclusters \leftarrow \emptyset$ \tcp*{New candidate
              features}
            \ParForAll{$p \in \set{1, \ldots, P}$}{ \tcp{Process local
                data} \For{$i \in \mathcal{B}(p,t)$} { \tcp{Optimistic
                  Transaction} \For{$f_k\in\clusters$}{ Set $z_{ik}$
                  to minimize $\|x_i - \sum_j z_{ij} f_j\|_{2}^{2}$ }
                \If{$ \|x_i-\sum_j z_{ij} f_j\|_2^2 > \lambda^2$}{
                  $f_i^{new} \leftarrow x_i-\sum_j z_{ij} f_j$\; $z_i
                  \leftarrow z_i \oplus$ \Ref{$f_i^{new}$} \;
                  $\gclusters \leftarrow \gclusters \union f_i^{new}$
                } } }
            \tcp{Serially validate features} $\clusters \leftarrow
            \clusters \,\, \union \,\, \DPValidate(\gclusters)$ \; }
          Compute $Z^TZ=\sum_iz_iz_i^T$ and $Z^TX=\sum_iz_ix_i^T$ in
          parallel\; Re-estimate features
          $F\leftarrow(Z^TZ)^{-1}Z^TX$\; } \Output{Accepted feature
          centers $\clusters$}
      \end{algorithm}
    }
  \end{figure}

\begin{figure}[t]
\centering
  {\small
      \begin{algorithm}[H]
        \DontPrintSemicolon
        \caption{Serial BP-means}
        \label{alg:bpm}
        \Input{data $\{x_i\}_{i=1}^N$, threshold $\lambda$}
        Initialize $z_{i1}=1$, $f_1 = N^{-1}\sum_i x_i$, $K=1$\;
        \While{not converged}{
          \For{$i$ = 1 to $N$}{
            \For{$k$ = 1 to $K$}{
              Set $z_{ik}$ to minimize $\|x_i - \sum_{j=1}^{K} z_{ij} f_j\|_{2}^{2}$
            }
            \If{$\|x_i - \sum_{j=1}^{K} z_{ij} f_{i,j}\|_{2}^{2} > \lambda^2$}{
              Set $K \leftarrow K+1$\;
              Create feature $f_K \leftarrow x_i - \sum_{k=1}^{K} z_{ik} f_j$\;
              Assign $z_{iK} \leftarrow 1$ (and $z_{i K} \leftarrow 0$ for $j \ne i$)\;
            }
          }
          $F \leftarrow (Z^{T} Z)^{-1} Z^{T} X$\;
        }
      \end{algorithm}
    }
  \end{figure}

  \begin{figure}
\centering
    {\small
      \begin{algorithm}[H]
        \DontPrintSemicolon
        \caption{\texttt{BPValidate} }
        \label{alg:bp_validate}
        \Input{Set of proposed feature centers $\gclusters$}
        $\clusters \leftarrow \emptyset$ \;
        \For{$f^{new} \in \gclusters$}{
            \For{$f_{k'}\in\clusters$}{
              Set $z_{ik'}$ to minimize $\|f^{new} - \sum_{f_j\in\clusters} z_{ij} f_j\|_{2}^{2}$
            }
          \If{$\|f^{new}-\sum_{f_j\in\clusters}z_{ij}f_j\|_2^2 > \lambda^2$}{
            $\clusters \leftarrow \clusters \union \set{f^{new}-\sum_{f_j\in\clusters}z_{ij}f_j}$
          }
          \Ref{$f^{new}$} $\leftarrow \set{z_{ij}}_{f_j\in\clusters}$
        }
        \Output{Accepted feature centers $\clusters$}
      \end{algorithm}
    }
    \end{figure}

\section{Proof of serializability of distributed algorithms}
\label{ssec:appendix_serialization}

\subsection{Proof of Theorem~\ref{thm:serial} for DP-means}

We note that both distributed DP-means and BP-means iterate over $z$-updates and cluster / feature means re-estimation until convergence.
In each iteration, distributed DP-means and BP-means perform the same set of updates as their serial counterparts.
Thus, it suffices to show that each iteration of the distributed algorithm is serially equivalent to an iteration of the serial algorithm.

Consider the following ordering on transactions:
\begin{itemize}
\item Transactions on individual data points are ordered before transactions that re-estimate cluster / feature means are ordered.
\item A transaction on data point $x_i$ is ordered before a transaction on data point $x_j$ if
\begin{enumerate}
\item $x_i$ is processed in epoch $t$, $x_j$ is processed in epoch $t'$, and $t<t'$
\item $x_i$ and $x_j$ are processed in the same epoch, $x_i$ and $x_j$ are not sent to the master for validation, and $i<j$
\item $x_i$ and $x_j$ are processed in the same epoch, $x_i$ is not sent to the master for validation but $x_j$ is
\item $x_i$ and $x_j$ are processed in the same epoch, $x_i$ and $x_j$ are sent to the master for validation, and the master serially validates $x_i$ before $x_j$
\end{enumerate}
\end{itemize}

We show below that the distributed algorithms are equivalent to the serial algorithms under the above ordering, by inductively demonstrating that the outputs of each transaction is the same in both the distributed and serial algorithms.

Denote the set of clusters after the $t$ epoch as $\clusters^t$.

The first transaction on $x_j$ in the serial ordering has $\clusters^0$ as its input.
By definition of our ordering, this transaction belongs the first epoch, and is either (1) not sent to the master for validation, or (2) the first data point validated at the master.
Thus in both the serial and distributed algorithms, the first transaction either (1) assigns $x_j$ to the closest cluster in $\clusters^0$ if $\min_{\mu_k\in\clusters^0}\bregd{x_j}{\mu_k}<\lambda$, or (2) creates a new cluster with center at $x_j$ otherwise.

Now consider any other transaction on $x_j$ in epoch $t$.
\begin{description}
\item\textbf{Case 1}: $x_j$ is not sent to the master for validation.

In the distributed algorithm, the input to the transaction is $\clusters^{t-1}$.
Since the transaction is not sent to the master for validation, we can infer that there exists $\mu_k\in\clusters^{t-1}$ such that $\bregd{x_j}{\mu_k}<\lambda$.

In the serial algorithm, $x_j$ is ordered after any $x_i$ if (1) $x_i$ was processed in an earlier epoch, or (2) $x_i$ was processed in the same epoch but not sent to the master (i.e. does not create any new cluster) and $i<j$.
Thus, the input to this transaction is the set of clusters obtained at the end of the previous epoch, $\clusters^{t-1}$, and the serial algorithm assigns $x_j$ to the closest cluster in $\clusters^{t-1}$ (which is less than $\lambda$ away).

\item\textbf{Case 2}: $x_j$ is sent to the master for validation.

In the distributed algorithm, $x_j$ is not within $\lambda$ of any cluster center in $\clusters^{t-1}$.
Let $\gclusters^t$ be the new clusters created at the master in epoch $t$ before validating $x_j$.
The distributed algorithm either (1) assigns $x_j$ to $\mu_{k^*}=\argmin_{\mu_k\in\gclusters^t}\bregd{x_j}{\mu_k}$ if $\bregd{x_j}{\mu_k}\leq\lambda$, or (2) creates a new cluster with center at $x_j$ otherwise.

In the serial algorithm, $x_j$ is ordered after any $x_i$ if (1) $x_i$ was processed in an earlier epoch, or (2) $x_i$ was processed in the same epoch $t$, but $x_i$ was not sent to the master (i.e. does not create any new cluster), or (3) $x_i$ was processed in the same epoch $t$, $x_i$ was sent to the master, and serially validated at the master before $x_j$.
Thus, the input to the transaction is $\clusters^{t-1}\cup\gclusters^t$.
We know that $x_j$ is not within $\lambda$ of any cluster center in $\clusters^{t-1}$, so the outcome of the transaction is either (1) assign $x_j$ to $\mu_{k^*}=\argmin_{\mu_k\in\gclusters^t}\bregd{x_j}{\mu_k}$ if $\bregd{x_j}{\mu_k}\leq\lambda$, or (2) create a new cluster with center at $x_j$ otherwise.
This is exactly the same as the distributed algorithm.
\end{description}

\subsection{Proof of Theorem~\ref{thm:serial} for BP-means}

The serial ordering for BP-means is exactly the same as that in DP-means.
The proof for the serializability of BP-means follows the same argument as in the DP-means case, except that we perform feature assignments instead of cluster assignments.

\subsection{Proof of Theorem~\ref{thm:serial} for OFL}
Here we prove Theorem~\ref{thm:serial} that the distributed OFL algorithm is equivalent to a serial algorithm.
\begin{proof}[(Theorem~\ref{thm:serial}, OFL)]
  We show that with respect to the returned centers (facilities), the distributed OFL algorithm is equivalent to running the serial OFL algorithm on a particular permutation of the input data.
  We assume that the input data is randomly permuted and the indices $i$ of the points $x_i$ refer to this permutation. We assign the data points to processors by assigning the first $b$ points to processor $p_1$, the next $b$ points to processor $p_2$, and so on, cycling through the processors and assigning them batches of $b$ points, as illustrated in Figure~\ref{fig:ofl_serialization_order}. In this respect, our ordering is generic, and can be adapted to any assignments of points to processors. We assume that each processor visits its points in the order induced by the indices, and likewise the master processes the points of an epoch in that order.
  
  \begin{figure}
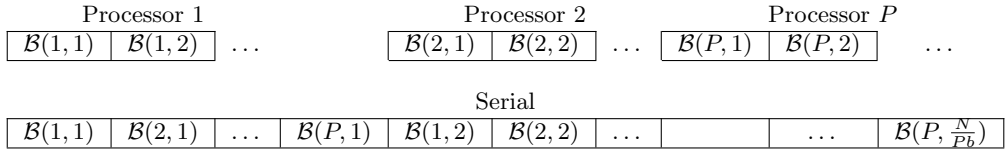

    \centering
    \hspace{-30pt}
    \begin{minipage}{1.0\linewidth}
      {\small
        \begin{tabular}{|c|c|cc|c|c|c|c|c|c}
          \multicolumn{3}{c}{Processor 1} &  \multicolumn{1}{c}{} & 
          \multicolumn{3}{c}{Processor 2} & 
          \multicolumn{3}{c}{Processor $P$} \\
          \cline{1-2} \cline{5-6} \cline{8-9}
          $\mathcal{B}({1,1})$ & $\mathcal{B}({1,2})$ & \ldots & \hspace{5pt} &
          $\mathcal{B}({2,1})$ & $\mathcal{B}({2,2})$ & \ldots & 
          $\mathcal{B}({P,1})$ & $\mathcal{B}({P,2})$ & \ldots \\
          \cline{1-2} \cline{5-6} \cline{8-9}
          \multicolumn{10}{c}{} \\
          \multicolumn{10}{c}{Serial} \\
          \cline{1-10}
          $\mathcal{B}({1,1})$ & $\mathcal{B}({2,1})$ & \ldots & \multicolumn{1}{|c|}{$\mathcal{B}({P,1})$} & $\mathcal{B}({1,2})$ & $\mathcal{B}({2,2})$ & \ldots & 
          & \multicolumn{1}{c}{\ldots} & \multicolumn{1}{|c|}{$\mathcal{B}(P,\frac{N}{Pb})$}\\
          \cline{1-10}
        \end{tabular}
      }
    \end{minipage}

    \caption{Illustration of distributed and serial order of blocks $\mathcal{B}({i,t})$ of length $b$ for OFL. The order within each block is maintained. Block $\mathcal{B}({i,t})$ is processed in epoch $t$ by processor $p_i$.}
    \label{fig:ofl_serialization_order}
  \end{figure}

  For the serial algorithm, we will use the following ordering of the data:
  Point $x_i$ precedes point $x_j$ if
  \begin{enumerate}
  \item $x_i$ is processed in epoch $t$ and $x_j$ is processed in epoch $t'$, and $t < t'$, or
  \item $x_i$ and $x_j$ are processed in the same epoch and $i < j$.
  \end{enumerate}
  If the data is assigned to processors as outlined above, then the serial algorithm will process the points exactly in the order induced by the indices. That means the set of points processed in any given epoch $t$ is the same for the serial and distributed algorithm.
We denote by $\clusters^{t}$ the global set of validated centers collected by OCC OFL up to (including) epoch $t$, and by $\newclusters^{i}$ the set of centers collected by the serial algorithm up to (including) point $x_i$.

  We will prove the equivalence inductively.

  \paragraph{Epoch $t=1$.} In the first epoch, all points are sent to the master. These are the first $Pb$ points. Since the master processes them in the same order as the serial algorithm, the distributed and serial algorithms are equivalent.

  \paragraph{Epoch $t>1$.} Assume that the algorithms are equivalent up to point $x_{i-1}$ in the serial order, and point $x_i$ is processed in epoch $t$. By assumption, the set $\clusters^{t-1}$ of global facilities for the distributed algorithm is the same as the set $\newclusters^{(t-1)Pb}$ collected by the serial algorithm up to point $x_{(t-1)Pb}$.
  For notational convenience, let $D(x_i,\clusters^{t}) = \min_{\mu \in \clusters^t}D(x_i,\mu)$ be the distance of $x_i$ to the closest global facility.

  The essential issue to prove is the following claim:
  \begin{claim}
    If the algorithms are equivalent up to point $x_{i-1}$, then 
    the probability of $x_i$ becoming a new facility is the same for the distributed and serial algorithm.
  \end{claim}
  The serial algorithm accepts $x_i$ as a new facility with probability $\min\{1,D(x_i,\newclusters^{i-1})/\lambda^2\}$.
  The distributed algorithm sends $x_i$ to the master with probability $\min\{1,D(x_i,\clusters^{t-1})\}$. The probability of ultimate acceptance (validation) of $x_i$ as a global facility is the probability of being sent to the master \emph{and} being accepted by the master.
  In epoch $t$, the master receives a set of candidate facilities with indices between $(t-1)Pb+1$ and $tPb$.
  It processes them in the order of their indices, i.e., all candidates $x_j$ with $j < i$ are processed before $i$.
  Hence, the assumed equivalence of the algorithms up to point $x_{i-1}$ implies that, when the master processes $x_i$, the set $\clusters^{t-1} \union \gclusters$ equals the set of facilities $\newclusters^{i-1}$ of the serial algorithm.
  The master consolidates $x_i$ as a global facility with probability $1$ if $D(x_i, \newclusters^{i-1} \union \gclusters) > \lambda^2$ and with probability 
 $D(x_i, \newclusters^{i-1} \union \gclusters) / D(x_i,\clusters^{t-1})$ otherwise. 

We now distinguish two cases. If the serial algorithm accepts $x_i$ because $D(x_i,\newclusters^{i-1}) \geq \lambda^2$, then for the distributed algorithm, it holds that
 \begin{align}
   D(x_i,\clusters^{t-1}) \geq D(x_i,\clusters^{t-1} \union \gclusters) = D(x_i,\newclusters^{i-1}) \geq \lambda^2
 \end{align}
 and therefore the distributed algorithm also always accepts $x_i$.

Otherwise, if $D(x_i,\newclusters^{i-1}) < \lambda^2$, then the serial algorithm accepts with probability $D(x_i,\newclusters^{i-1})/\lambda^2$. The distributed algorithm accepts with probability
  \begin{align}
     \mathbb{P}(x_i \text{ accepted}) &= \mathbb{P}(x_i \text{ sent to master } )\cdot \mathbb{P}(x_i \text{ accepted at master})\\
     &= \frac{D(x_i,\clusters^{t-1})}{\lambda^2} \cdot \frac{D(x_i,\newclusters^{i-1} \union \gclusters)}{D(x_i,\clusters^{t-1})}\\
     &= \frac{D(x_i,\newclusters^{i-1})}{\lambda^2}.
  \end{align}
This proves the claim.

The claim implies that if the algorithms are equivalent up to point $x_{i-1}$, then they are also equivalent up to point $x_i$. This proves the theorem.
\end{proof}

\subsection{Proof of Lemma~\ref{lem:approxbd} (Approximation bound)}
\label{app:approxbd}

We begin by relating the results of facility location algorithms and DP-means.
Recall that the objective of DP-means and FL is
\begin{equation}
  \label{eq:FL_objective}
  J(\clusters) = \sum_{x \in X}\min_{\mu \in \clusters}\bregd{x}{\mu}^2 + \lambda^2 |\clusters|.
\end{equation}
In FL, the facilities may only be chosen from a pre-fixed set of centers (e.g., the set of all data points), whereas DP-means allows the centers to be arbitrary, and therefore be the empirical mean of the points in a given cluster.
  However, choosing centers from among the data points still gives a factor-2 approximation. Once we have established the corresponding clusters, shifting the means to the empirical cluster centers never hurts the objective. The following proposition has been a useful tool in analyzing clustering algorithms:
\begin{prop}\label{prop:FLtoDP}
  Let $\clusters^*$ be an optimal solution to the DP-means problem~\eqref{eq:FL_objective}, and let $\clusters^{\text{FL}}$ be an optimal solution to the corresponding FL problem, where the centers are chosen from the data points. Then
  \begin{equation*}
    J(\clusters^{\text{FL}}) \leq 2 J(\clusters^*).
  \end{equation*}
\end{prop}
\begin{proof}{\emph{(Proposition~\ref{prop:FLtoDP})}}
  It is folklore that Proposition~\ref{prop:FLtoDP}) holds for the K-means objective, i.e.,
  \begin{equation}
    \label{eq:1}
    \min_{\clusters \subseteq X, |\clusters|=k} \sum_{i=1}^n\min_{\mu \in \clusters}\|x_i - \mu\|^2 \leq 2 \min_{\clusters \subseteq X} \sum_{i=1}^n\min_{\mu \in \clusters}\|x_i - \mu\|^2.
  \end{equation}
  In particular, this holds for the optimal number $K^* = |\clusters^*|$. Hence, it holds that 
  \begin{equation}
    \label{eq:2}
    J(\clusters^{\text{FL}}) \leq \min_{\clusters \subseteq X, |\clusters|=K^*} \sum_{i=1}^n\min_{\mu \in \clusters}\|x_i - \mu\|^2 + \lambda^2 K^* \leq 2 J(\clusters^*).
  \end{equation}
\end{proof}

With this proposition at hand, all that remains is to prove an approximation factor for the FL problem.
\begin{proof}{\emph{(Lemma~\ref{lem:approxbd})}}
  First, we observe that the proof of Theorem~\ref{thm:serial} implies that, for any random order of the data, the OCC and serial algorithm process the data in exactly the same way, performing ultimately exactly the same operations. Therefore, any approximation factor that holds for the serial algorithm straightforwardly holds for the OCC algorithm too.

  Hence, it remains to prove the approximation factor of the serial algorithm. 
  Let $C^{\text{FL}}_1, \ldots, C^{\text{FL}}_k$ be the clusters in an optimal solution to the FL problem, with centers $\mu^{\text{FL}}_1, \ldots \mu^{\text{FL}}_k$. We analyze each optimal cluster individually. The proof follows along the lines of the proofs of Theorems 2.1 and 4.2 in \citep{Meyerson01}, adapting it to non-metric squared distances. We show the proof for the constant factor, the logarithmic factor follows analogously by using the ring-splitting as in \citep{Meyerson01}.
  
First, we see that the expected total cost of any point $x$ is bounded by the distance to the closest open facility $y$ that is present when $x$ arrives. If we always count in the distance of $\|x-y\|^2$ into the cost of $x$, then the expected cost is $\gamma(x) = \lambda^2 \|x-y\|^2/\lambda^2 + \|x-y\|^2 = 2\|x-y\|^2$.

We consider an arbitrary cluster $C^*_i$ and divide it into $|C^*|/2$ \emph{good} points and $|C^*|/2$ \emph{bad points}. Let $D_i = \frac{1}{|\clusters^{\text{FL}}|}\sum_{x \in C^*_i}\|x - \mu_i\|$ be the average service cost of the cluster, and let $d_g$ and $d_b$ be the service cost of the good and bad points, respectively (i.e., $D_i = (d_g + d_b)/|\CFL_i|$). The good points satisfy $\|x-\mu^{\text{FL}}_i\| \leq 2 D_i$.
Suppose the algorithm has chosen a center, say $y$, from the points $\CFL_i$. Then any other point $x \in \CFL_i$ can be served at cost at most
\begin{align}\label{eq:oflproof1}
  \|x - y\|^2 \leq \Big(\|x-\muFL_i\| + \|y - \muFL_i\|\Big)^2 \leq 2 \|x - \muFL_i\|^2 + 4D_i.
\end{align}
That means once the algorithm has established a good center within $\CFL_i$, all other good points together may be serviced within a constant factor of the total optimal service cost of $\CFL$, i.e., at $2d_g + 4(d_g + d_b)$.
The assignment cost of all the good points in $\CFL_i$ that are passed before opening a good facility is, by construction of the algorithm and expected waiting times, in expectation $\lambda^2$. Hence, in expectation, the cost of the good points in $\CFL_i$ will be bounded by $\sum_{x \text{good}}\gamma(x) \leq 2(2d_g + 4d_g + 4d_b + \lambda^2)$.

Next, we bound the expected cost of the bad points. We may assume that the bad points are injected randomly in between the good points, and bound the servicing cost of a bad point $x_b \in \CFL_i$ in terms of the closest good point $x_g \in \CFL_i$ preceding it in our data sequence. Let $y$ be the closest open facility to $\muFL_i$ when $y$ arrives. Then
\begin{align}\label{eq:oflproof2}
  \|x_b - y\|^2 \leq 2\|y - \muFL_i\|^2 + 2\|x_b - \muFL\|^2.
\end{align}
Now assume that $x_g$ was assigned to $y'$. Then
\begin{align}\label{eq:oflproof3}
  \|y - \muFL_i\|^2 \leq \|y' - \muFL_i\|^2 \leq 2\|y' - x_g\|^2 + 2\|x_g - \muFL\|^2.
\end{align}
From \eqref{eq:oflproof2} and \eqref{eq:oflproof1}, it then follows that 
\begin{align}\label{eq:oflproof4}
   \|x_b - y\|^2 &\leq 4\|y' - x_g\|^2 + 4\|x_g - \muFL\|^2 + 2\|x_b - \muFL\|^2\\
   &= 2\gamma(x_g) + 4\|x_g - \muFL\|^2 + 2\|x_b - \muFL\|^2.
\end{align}
Since the data is randomly permuted, $x_g$ could be, with equal probability, any good point, and in expectation we will average over all good points.

Finally, with probability $2/|\CFL_i|$ there is no good point before $x_g$. In that case, we will count in $x_b$ as the most costly case of opening a new facility, incurring cost $\lambda^2$.
In summary, we can bound the expected total cost of $\CFL$ by
\begin{align}
  \nonumber
  &\sum_{x \text{ good}} \gamma(x) +   \sum_{x \text{ bad}} \gamma(x)\\
  &\leq 12 d_g + 8 d_b + \lambda^2 + \frac{2\CFL}{2\CFL}\lambda^2 + 2(2\frac{2|\CFL_i|}{2|\CFL|} (12 d_g + 8 d_b + \lambda^2) + 4d_g + 2d_b)\\
  &\leq 68 d_g + 42d_b+ 4\lambda^2 \leq 68 J(\clusters^{\text{FL}}).
\end{align}
This result together with Proposition~\ref{prop:FLtoDP} proves the lemma.
\end{proof}

\section{Proof of master processing bound for DP-means (Theorem~\ref{thm:dp_rate})}
\label{ssec:appendix_dpmeans}

\begin{proof}
As in the theorem statement, we assume $P$ processors, $b$ points assigned
to each processor per epoch, and $N$ total data points. We further assume
a generative model for the cluster memberships: namely, that they are generated
iid from an arbitrary distribution $(\pi_j)_{j=1}^{\infty}$. That is, we have
$\sum_{j=1}^{\infty} \pi_{j} = 1$ and, for each $j$, $\pi_{j} \in [0,1]$. We see that
there are perhaps infinitely many latent clusters. Nonetheless, in any data set of finite
size $N$, there will of course be only finitely many clusters to which
any data point in the set belongs. Call the number of
such clusters $K_{N}$.

Consider any particular cluster indexed by $j$. At the end of the first epoch in which
a worker sees $j$, that worker (and perhaps other workers) will send some data point
from $j$ to the master. By construction, some data point from $j$ will belong to the
collection of cluster centers at the master by the end of the processing done
at the master and therefore by the beginning of the next epoch. It follows
from our assumption (all data points within a single cluster are within a $\lambda$ diameter)
that no other data point from cluster $j$ will be sent to the master in future epochs.
It follows from our assumption about the separation of clusters that no points
in other clusters will be covered by any data point from cluster $j$.

Let $S_{j}$ represent the (random) number of points from cluster $j$ sent to
the master.
Since there are $Pb$ points processed by workers in a single epoch, $N_{j}$
is constrained to take values between $0$ and $Pb$. Further, note that
there are a total of $N/(Pb)$ epochs.

Let $A_{j,s,t}$ be the event that the master is sent $s$ data points from cluster
$j$ in epoch $t$. All of the events $\{A_{j,s,t}\}$
with $s = 1,\ldots,Pb$ and $t=1,\ldots,N/(Pb)$ are disjoint. 
Define $A'_{j,0}$ to be the event that, for all epochs $t=1,\ldots,N/(Pb)$, zero
data points are sent to the master; i.e., $A'_{j,0} := \bigcup_{t} A_{j,0,t}$. Then
$A'_{j,0}$ is also disjoint from the events $\{A_{j,s,t}\}$
with $s = 1,\ldots,Pb$ and $t=1,\ldots,N/(Pb)$. Finally,
$$
	A'_{j,0} \cup \bigcup_{s=1}^{Pb} \bigcup_{t=1}^{N/(Pb)} A_{j,s,t}
$$
covers all possible data configurations.
It follows that
\begin{align*}
	\mathbb{E}[ S_{j} ]
		&= 0 * \mathbb{P}[A'_{j,0}] + \sum_{s=1}^{Pb} \sum_{t=1}^{N/(Pb)} s \mathbb{P}[A_{j,s,t}]
		= \sum_{s=1}^{Pb} \sum_{t=1}^{N/(Pb)} s \mathbb{P}[A_{j,s,t}]
\end{align*}

Note that, for $s$ points from cluster $j$ to be sent to the master at epoch $t$, it must be the case
that no points from cluster $j$ were seen by workers during epochs $1, \ldots, t-1$, and then
$s$ points were seen in epoch $t$. That is,
$
	\mathbb{P}[A_{j,s,t}]
		= (1-\pi_{j})^{Pb(t-1)} \cdot \binom{Pb}{s} \pi_{j}^{s} (1-\pi_{j})^{Pb - s}.
$

Then
\begin{align*}
	\mathbb{E}[ S_{j} ]
		&= \left( \sum_{s=1}^{Pb} s \binom{Pb}{s} \pi_{j}^{s} (1-\pi_{j})^{Pb - s} \right)
			\cdot \left( \sum_{t=1}^{N/(Pb)} (1-\pi_{j})^{Pb(t-1)} \right) \\
		&= \pi_{j} Pb \cdot \frac{1 - (1-\pi_{j})^{Pb \cdot N/(Pb)}}{1- (1-\pi_{j})^{Pb}},
\end{align*}
where the last line uses the known, respective
forms of the expectation of a binomial random variable
and of the sum of a geometric series.

To proceed, we make use of a lemma.
\begin{lem}
Let $m$ be a positive integer and $\pi \in (0,1]$. Then
$$
	\frac{1}{1-(1-\pi)^{m}} \le \frac{1}{m \pi} + 1.
$$
\end{lem}

\begin{proof}
A particular subcase of Bernoulli's inequality tells us that, for integer $l \le 0$
and real $x \ge -1$,
we have $(1+x)^l \ge 1 + lx$. Choose $l = -m$ and $x = -\pi$. Then
\begin{align*}
	(1-\pi)^{m}
		&\le \frac{1}{1 + m\pi} \\
	\Leftrightarrow 1 - (1-\pi)^m
		&\geq 1 - \frac{1}{1 + m \pi} = \frac{m \pi}{1 + m \pi}\\
	\Leftrightarrow \frac{1}{1 - (1-\pi)^m}
		&\leq \frac{m \pi + 1}{ m \pi} = \frac{1}{m \pi} + 1.
\end{align*}
\end{proof} 

We can use the lemma to find the expected total number of data points sent
to the master:
\begin{align*}
	\mathbb{E} \sum_{j=1}^{\infty} S_{j}
		&= \sum_{j=1}^{\infty} \mathbb{E}S_{j} 
			= \sum_{j=1}^{\infty} \pi_{j} Pb \cdot \frac{1 - (1-\pi_{j})^{N}}{1- (1-\pi_{j})^{Pb}} \\
		&\le \sum_{j=1}^{\infty} \pi_{j} Pb \cdot \left(1 + \frac{1}{\pi_{j} Pb} \right)
			\cdot \left( 1 - (1-\pi_{j})^{N} \right) \\
		&= Pb \sum_{j=1}^{\infty} \pi_{j} \left( 1 - (1-\pi_{j})^{N} \right) + \sum_{j=1}^{\infty} \left( 1 - (1-\pi_{j})^{N} \right) \\
		&\le Pb + \sum_{j=1}^{\infty} \mathbb{P}(\textrm{cluster $j$ occurs in the first $N$ points}) \\
		&= Pb + \mathbb{E}[K_{N}].
\end{align*}

Conversely,
\begin{align*}
	\mathbb{E} \sum_{j=1}^{\infty} S_{j}
		&\ge \sum_{j=1}^{\infty} \pi_{j} Pb = Pb.
\end{align*}
\end{proof}

\subsection{Experiment}

To demonstrate the bound on the expected number of data points proposed but not accepted as new centers, we generated synthetic data with separable clusters.
Cluster proportions are generated using the stick-breaking procedure for the Dirichlet process, with concentration parameter $\theta=1$.
Cluster means are set at $\mu_k = (2 k,0,0,\dots,0)$, and generated data uniformly in a ball of radius $1/2$ around each center.
Thus, all data points from the same cluster are at most distance $1$ from one another, and more than distance of $1$ from any data point from a different cluster.

We follow the same experimental framework in Section \ref{ssec:evaluation_simulated}.

\begin{figure}[h]
  \centering
  \begin{subfigure}[b]{0.48\textwidth}
  	\centering
	  \includegraphics[width=200pt]{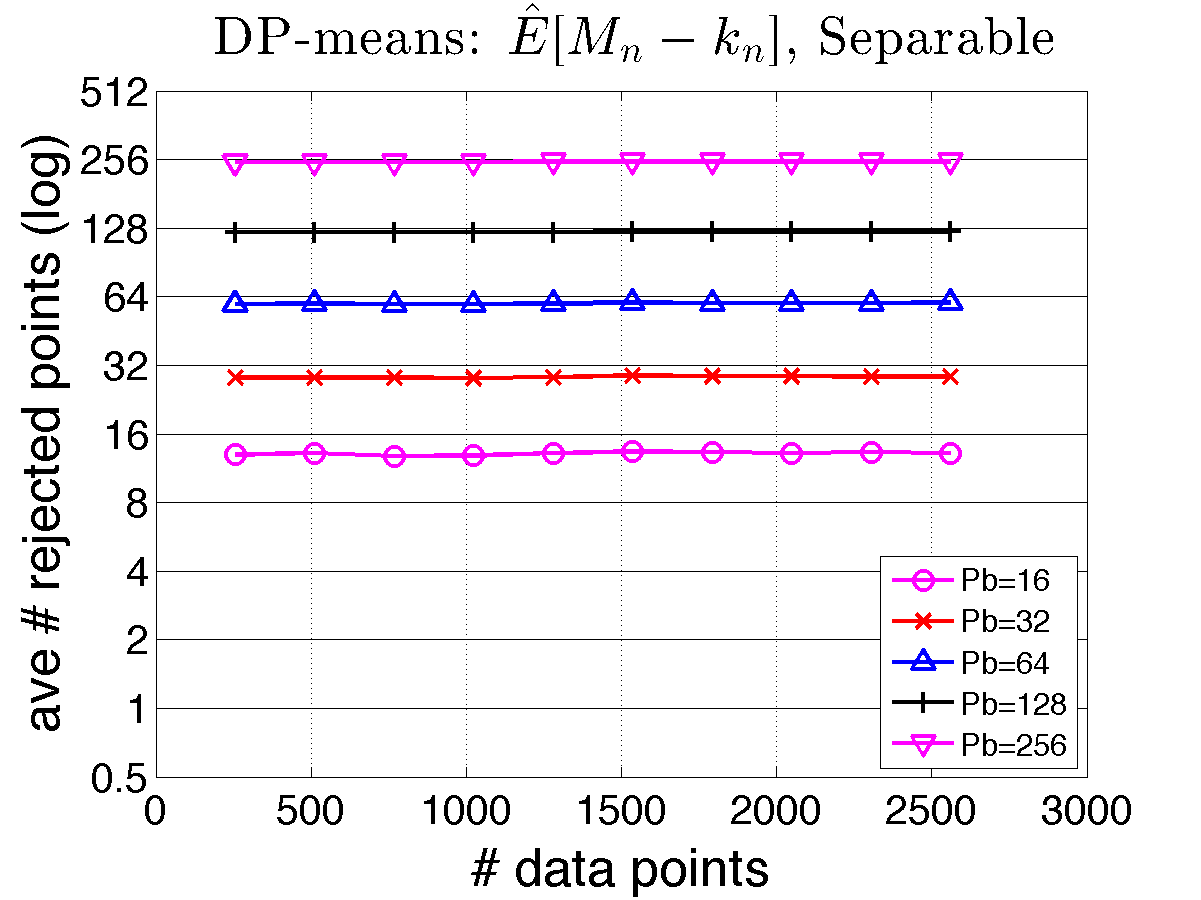}
  	\caption{\footnotesize DP-means, separable}
	  \label{fig:dpm_sim_coupon1}
	\end{subfigure}
  \begin{subfigure}[b]{0.48\textwidth}
  	\centering
	  \includegraphics[width=200pt]{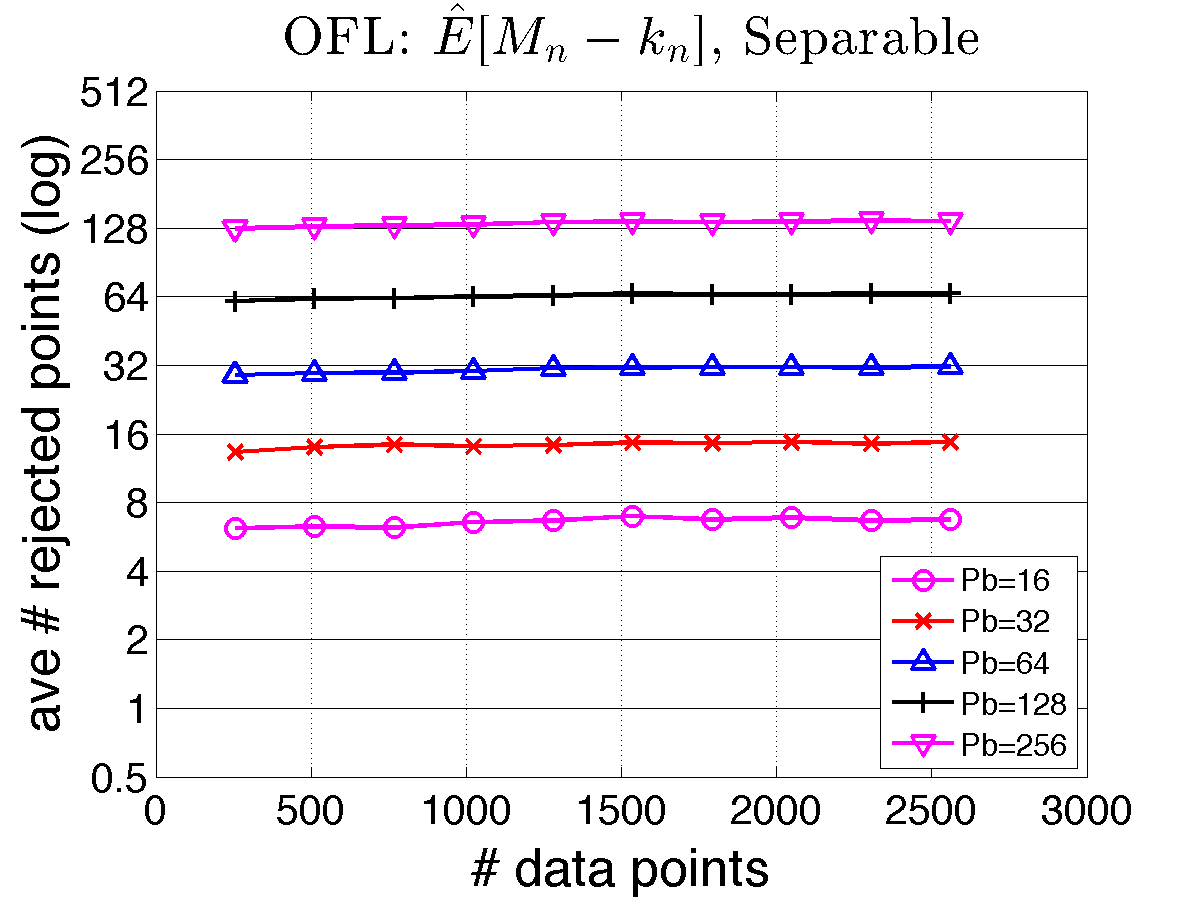}
  	\caption{\footnotesize OFL, separable}
	  \label{fig:ofl_sim_coupon1}
	\end{subfigure}
	\caption{\footnotesize Simulated distributed DP-means and OFL: expected number of data points proposed but not accepted as new clusters is independent of size of data set.}
	\label{fig:dpmofl_sim_coupon1}
\end{figure}

In the case where we have separable clusters (Figure \ref{fig:dpmofl_sim_coupon1}), $\hat{\mathbb{E}}[M_N-k_N]$ is bounded from above by $Pb$, which is in line with the above Theorem \ref{thm:dp_rate}.

\end{document}